\documentclass{article}

\usepackage[main, final]{neurips_2026}
\usepackage[utf8]{inputenc} % allow utf-8 input
\usepackage[T1]{fontenc}    % use 8-bit T1 fonts
\usepackage{hyperref}       % hyperlinks
\usepackage{url}            % simple URL typesetting
\usepackage{booktabs}       % professional-quality tables
\usepackage{amsfonts}       % blackboard math symbols
\usepackage{nicefrac}       % compact symbols for 1/2, etc.
\usepackage{microtype}      % microtypography
\usepackage{xcolor}         % colors
\usepackage{amsmath, amssymb, amsthm}
\usepackage{mathtools}
\usepackage{pifont}

\usepackage{algorithm}
\usepackage{algpseudocode}

\theoremstyle{plain}
\newtheorem{theorem}{Theorem}[section]
\newtheorem{proposition}[theorem]{Proposition}
\newtheorem{lemma}[theorem]{Lemma}
\newtheorem{corollary}[theorem]{Corollary}
\theoremstyle{definition}

\newtheorem{assumption}[theorem]{Assumption}
\theoremstyle{remark}
\newtheorem{remark}[theorem]{Remark}
\newtheorem{claim}[theorem]{Claim}

\DeclareMathOperator*{\E}{\mathbb{E}}
\DeclareMathOperator*{\argmax}{arg\,max}
\newcommand{\NR}{\mathrm{NR}}
\newcommand{\NRtilde}{\widetilde{\mathrm{NR}}}
\newcommand{\Ind}{\mathbf{1}}
\newcommand{\NCB}{\mathrm{NCB}}
\newcommand{\NRhat}{\widehat{\mathrm{NR}}}
\newcommand{\tildeO}{\Tilde{\mathcal{O}}}

\newcommand{\cmark}{\ding{51}}  % checkmark
\newcommand{\xmark}{\ding{55}}  % cross
\title{Nash Social Welfare for Multi Armed Bandits: Trajectory-wise Expected and High Probability Regret}

\author{%
  Avishek Ghosh \\
  Department of Computer Science and Engineering\\
  Indian Institute of Technology, Bombay\\
  \texttt{avishek\_ghosh@iitb.ac.in} \\
}

\begin{document}

\maketitle

\begin{abstract}
We study fair multi-armed bandits under the Nash Social Welfare (NSW) objective, which 
measures performance via the geometric mean of accumulated rewards as a fairness-respecting performance metric. Existing formulations 
define Nash regret as $\mathrm{NR}_T = \mu^\star - (\prod_{t=1}^T \mathbb{E}\mu_{I_t})^{1/T}$, 
where $\mu_{I_t}$ is the mean reward of the recommended arm $I_t$ and $T$ is the learning 
horizon. We observe that, since ensemble Nash regret operates on per-round marginal expectations before applying the geometric mean, it does not capture the joint distribution of rewards across rounds — leaving an aspect of the NSW fairness motivation unaddressed at the trajectory level.

To address this, we propose \emph{trajectory-wise Nash regret} $\widetilde{\mathrm{NR}}_T = 
\mu^\star - \mathbb{E}[(\prod_{t=1}^T \mu_{I_t})^{1/T}]$, where the geometric mean is 
computed over complete sample paths before taking expectations. This formulation captures 
the fairness properties of NSW more faithfully and requires controlling the joint distribution 
of rewards across rounds, rather than merely per-round marginals. By Jensen's inequality 
applied to the concave geometric mean, $\widetilde{\mathrm{NR}}_T \geq \mathrm{NR}_T$, 
making it a strictly stronger metric. We additionally introduce \emph{high probability Nash 
regret} $\widehat{\mathrm{NR}}_T = \mu^\star - (\prod_t \mu_{I_t})^{1/T}$, obtaining the 
first high probability regret bounds in the fair bandits literature. We propose a two-phase algorithm, Round Robin Nash Confidence Bound (\texttt{RR-NCB}), 
combining structured round robin exploration with a Nash confidence bound index policy. We 
establish that $\widetilde{\mathrm{NR}}_T \leq \widetilde{\mathcal{O}}(\sqrt{k\log T/T})$ 
and $\widehat{\mathrm{NR}}_T \leq \widetilde{\mathcal{O}}(\sqrt{k\log(kT/\delta)/T})$ with 
probability $1-\delta$, matching the optimal $\widetilde{\mathcal{O}}(\sqrt{k/T})$ scaling 
despite operating under strictly stronger metrics. Optimality follows from a lower bound 
chain via AM-GM and standard $k$-armed bandit minimax arguments. We validate our theoretical 
findings through simulations.
\end{abstract}

\section{Introduction}
\label{sec:intro}
We consider the canonical framework of multi-armed bandit where an agent is interacting with the environment with $k$ arms for $T$ rounds. At every time instant, the agent pulls an arm out of these and receives a random reward. The rewards received by the agent are iid across rounds with mean reward corresponding to arms $\{1,\ldots,k\}$ be denoted by $\mu_1, \ldots, \mu_k$ apriori unknown. We assume $\mu_i \in (0,1]$ for all $i \in [k]$ and $\mu^\star := \max_{i \in [k]} \mu_i$. Without loss of generality, fix $i^\star \in \argmax_i \mu_i$. Typically, to characterize the performance of a learning algorithm, one considers the notion of cumulative or average regret given by the difference between the mean reward of best arm $\mu^*$ and the arithmetic mean of accumulated rewards over rounds $ \frac{1}{T}\sum_{i=1}^T \E \mu_{I_t}$, where $\{\mu_{I_t}\}_{i=1}^T$ are the recommendations of the leaning algorithm (\cite{Lattimore_Szepesvari_2020,bubeck2012regret}).

However, in welfare economics, it is well understood that the arithmetic mean overlooks the fairness aspect completely (\cite{moulin,kaneko,KANEKO1981173,unreasonable}). This becomes critical in certain applications like clinical trials and drug recommendation. A natural alternative is to use axiomatically-justified welfare function namely Nash Social Welfare (NSW), where one takes geometric mean (instead of arithmetic mean) of the accumulated rewards. As discussed in \cite{moulin,unreasonable} NSW satisfies several fundamental axioms like scale invariance, symmetry, independence of unconcerned agents and the Pigou-Dalton transfer principle (Robinhood principle). In a nutshell, in our context, the Pigou-Dalton principle implies that NSW will increase in the event where a well-off arm (with high reward) transfers $\delta$ reward to a  low reward (value) arm. At the same time, if the relative increase in the value of low reward arm is significantly less than the relative decrease in the value of the well-off arm, NSW does not favor such a transfer, striking a balance between fairness and efficiency. In \cite{kaneko}, it is argued that NSW is the unique welfare function satisfying all the above-mentioned axioms.

To this end \cite{barman} introduces the notion of \emph{fair bandits}, where the goal is to minimize the Nash regret given by $\mathrm{NR}_T = \mu^\star - \![(\prod_{t=1}^T \E \mu_{I_t})^{1/T}] $. Note that this formulation respects the fairness aspect as it considers the geometric mean of the accumulated reward. A learning algorithm with (optimal) regret scaling of $\tildeO(\sqrt{K/T})$ is given for bounded (between $[0,1]$) rewards. Later \cite{sayak} extended this result to capture sub-Gaussian rewards and obtain similar regret. Moreover \cite{krishna2025p} study the $p$ mean regret, where the $\ell_p$ norm of the accumulated reward is considered as a fair metric. Note that when $p \rightarrow 0$, this metric boils down to the NSW.

In all the above mentioned papers, the Nash Social Welfare function is defined 
by taking the expectation over the algorithm's randomness \emph{first} and then 
applying the geometric mean. More formally, at every round $t$, one computes the 
ensemble average $\mathbb{E}[\mu_{I_t}]$ across sample paths, yielding a sequence 
of non-random quantities $(\mathbb{E}[\mu_{I_1}], \ldots, \mathbb{E}[\mu_{I_T}])$, 
to which the geometric mean is then applied. The distinction between this ensemble NSW 
and the arithmetic mean therefore reduces to the choice of aggregation function 
on this \emph{same} non-random sequence: $f_1(z_1,\ldots,z_T) = 
\frac{1}{T}\sum_{i=1}^T z_i$ versus $f_2(z_1,\ldots,z_T) = 
\bigl(\prod_{i=1}^T z_i\bigr)^{1/T}$. Since these are related by a log 
transform, $\log f_2(z_1,\ldots,z_T) = f_1(\log z_1,\ldots,\log z_T)$, 
optimizing $f_2$ over policies is equivalent to optimizing $f_1$ on 
log-transformed marginals. This suggests that, at the level of per-round 
marginal expectations, ensemble Nash regret and the arithmetic mean objective 
may share more structural similarity than the NSW motivation implies --- and 
crucially, neither captures the \emph{joint} distribution of rewards across 
rounds, which is where trajectory-level fairness resides. 

To address this, we propose a modified NSW that moves the expectation operator 
\emph{outside} the geometric mean: $
    \widetilde{\mathrm{NSW}}_T 
    = \mathbb{E} [(\prod_{t=1}^T \mu_{I_t})^{\!1/T}]$, where the geometric mean is computed over a complete sample path first, and then 
averaged across trajectories. Unlike the ensemble formulation, this requires 
controlling the joint distribution of $(\mu_{I_1}, \ldots, \mu_{I_T})$ rather 
than merely the per-round marginals $\mathbb{E}[\mu_{I_t}]$, and no deterministic 
log-transform reduces it to an arithmetic mean. 

By Jensen's inequality applied to the concave geometric mean function, we have $\widetilde{\mathrm{NSW}}_T \leq \mathrm{NSW}_T$. Since we look at the complete trajectory, the corresponding regret is defined as Trajectory-wise Nash regret and the we call the original Nash regret as Ensemble Nash regret. We have, 
\[
   \NRtilde_T \;=\; \mu^\star \;-\; \E\![(\prod_{t=1}^T \mu_{I_t})^{\!1/T}] \geq \mathrm{NR}_T.
\]
$\NRtilde_T$ is a stronger metric compared to $\mathrm{NR}_T$ and upper bounds on $\mathrm{NR}_T$ does not translate.

Moreover, in this paper we also study the random variable $(\prod_{t=1}^T \mu_{I_t})^{\!1/T}$ trajectory-wise and obtain high probability regret bound. There are important settings where  averaging is not appropriate. In a clinical trial, for instance, each patient population experiences a single, irreversible sequence of treatment decisions — there is no ensemble of parallel runs to average over. What matters is the welfare of the actual trajectory, not its expectation. To the best of our knowledge, this is the first work to obtain a high probability regret on fair bandits. We define
\begin{align*}
    \NRhat_T := \mu^\star - \bigl(\prod_t \mu_{I_t}\bigr)^{1/T}\!.
\end{align*}

\subsection{Summary of Contributions}
We summarize our contribution in Table~\ref{table:comp}. We mention the details here.
\begin{table}[t!]
\label{table:comp}
\centering
\begin{tabular}{lcccc}
\toprule
& \textbf{Ensemble} & \textbf{Trajectory-wise} & \textbf{High Prob.} & \textbf{Regret} \\
& \textbf{Regret}   & \textbf{Regret}          & \textbf{Regret}    & \textbf{Scaling}\\
\midrule
\cite{barman}  & \cmark & \xmark & \xmark & $\Tilde{\mathcal{O}}(\sqrt{K/T})$ \\
\cite{krishna2025p}  & \cmark & \xmark & \xmark & $\Tilde{\mathcal{O}}(\sqrt{K/T})$ \\
\cite{sayak}   & \cmark & \xmark & \xmark & $\Tilde{\mathcal{O}}(\sqrt{K/T})$ \\
This paper           & \cmark & \cmark & \cmark & $\Tilde{\mathcal{O}}(\sqrt{K/T})$ \\
\bottomrule
\end{tabular}
\caption{Comparison of regret guarantees across methods. \cite{barman,sayak,krishna2025p} only works with ensemble regret, whereas we explicitly handle trajectory-wise and high probability regret. Since $\NRtilde_T \geq \mathrm{NR}_T$, our framework also gives an upper bound for ensemble regret.}
\end{table}

\begin{itemize}
\item \emph{Matching Regret Upper Bound:} We propose a learning algorithm, namely Round Robbin Nash Confidence Bound (\texttt{RR-NCB}) which uses round robbin exploration along with an index policy governed by the Nash confidence bound (see Algorithm~\ref{algo:main}) for details). We show that 
\[
\NRtilde_T \leq  \tildeO\!\bigl(\sqrt{\tfrac{k\log T}{T}} \bigr)
\quad\text{and}\quad
\Pr\!\Bigl[\NRhat_T \le \tildeO\!\bigl(\sqrt{\tfrac{k\log(8kT/\delta)}{T}}\bigr) \Bigr] \ge 1 - \delta,
\]
where $\tildeO(.)$ hides other log factors. Note that the regret scaling is $\tildeO(\sqrt{k/T})$ (ignoring log factor), which matches the regret upper bound of \cite{barman}, \cite{sayak} and \cite{krishna2025p} (for $p \rightarrow 0$). This implies that although $\NRtilde_T \geq \mathrm{NR}_T$, we were able to obtain the same regret scaling. Hence, the trajectory-wise Nash regret, which is a stronger metric comes with no additional cost in terms of regret. The same conclusion holds for high probability regret bound as well. To the best our knowledge no previous work has obtained high probability Nash regret.

\item \emph{Matching Regret Lower Bound:} Invoking the Jensen's inequality, we have the following chain of inequalities:
\begin{align*}
    \NRtilde_T \geq \mathrm{NR}_T \geq \mu^* - \frac{1}{T}\sum_{t=1}^T \E \mu_{I_{I_t}} \geq \Omega(\sqrt{k/T}),
\end{align*}
where the second inequality comes from AM-GM inequality and the third inequality is the standard lower bound for $k$-armed bandits (see \cite{Lattimore_Szepesvari_2020,bubeck2012regret}). Hence, the regret achieved by \texttt{RR-NCB} is optimal. Similar argument holds for the high probability bound as well. We have
\begin{align*}
    \NRhat_T \geq \mu^* - \frac{1}{T}\sum_{t=1}^T \mu_{I_{I_t}} \geq \Omega \bigl(\sqrt{k/T}\log(1/\delta)\bigr),
\end{align*}
with probability at least $1-\delta$. Note that this is a sample path-wise (trajectory-wise) regret. Hence \texttt{RR-NCB} yields optimal regret for both Trajectory-wise  $\NRtilde_T$ and high probability $\NRhat_T$ Nash regret.

\item \emph{Simulations:} We validate out theoretical finding through simulations. In particular we show that trajectory-wise average $\NRtilde_T$ yields worse regret compared to $\mathrm{NR}_T$ as predicted by theory. However, we show that by playing \texttt{RR-NCB}, the gap can be made small. The details are deferred to Section~\ref{sec:sim}. 
\end{itemize}

\subsection{Technical Novelty}
The analysis of Trajectory-wise and high probability Nash regret comes with additional technical challenges. Here we succinctly summarize a few of them.

\begin{itemize}
\item  \emph{Trajectory-wise Jensen-on-exp identity:} The geometric mean satisfies an \emph{exact} trajectory wise identity:
\[
   \Bigl(\textstyle\prod_t \mu_{I_t}\Bigr)^{\!1/T} \;=\; \mu^\star\cdot \exp(-\widehat R_T) \quad \text{on every sample path,}
\]
where $\widehat R_T = (1/T)\sum_t\log(\mu^\star/\mu_{I_t})$ is the realized cumulative log-regret. Now, we use Jensen's inequality on the exponential function. Combined with the identity $1 - e^{-x} \le x$, we obtain $\NRtilde_T \le \mu^\star\bar R_T$ where $\bar R_T = \E[\widehat R_T]$.

On the other hand, for high probability regret, we can skip the Jensen step and directly invoke $1 - e^{-x} \le x$ to obtain $\NRhat_T \le \mu^\star \widehat R_T$. Hence, the high-probability analysis of $\NRhat_T$ is cleaner than the in-expectation analysis of $\NRtilde_T$. By contrast, controlling $\widehat{NR}_T$ via concentrating $(\prod_t \mu_{I_t})^{1/T}$ around its expectation (the naive route) would be substantially more involved.

\item  \emph{Unified analysis and Log regret as unifying intermediate:} We argue that both regret are reduced to the same intermediate object — the cumulative log-regret. This enables us to give a unified analysis that holds for both trajectory-wise and high probability Nash regret.

\item  \emph{Round Robbin Exploration:} In \cite{barman,sayak}, thanks to the expectation inside geometric mean, a simple uniform exploration gives $\E \mu_{I_t} \geq \mu^*/k$ trivially. However, such a lower bound cannot be applied in out setup. Hence, in our learning algorithm \texttt{RR-NCB}, we design the first phase as round robbin exploration. We design the length of such exploration carefully and include a stopping criteria so that every arm gets pulled sufficiently and at the same time the length of exploration period is not too large (otherwise the regret would be impacted). 

\item \emph{Good Event and bad arms:} We define a good event $E$ judiciously and show (in Lemma~\ref{lem:lem11} that on  $E$, bad arms (with mean reward $\mu_j \le \mu^\star/64$) are never played in Phase 2. This crucially controls the regret in Phase 2. Moreover, under $E$, the mean estimate of the good arms (with $\mu_j > \mu^\star/64$) stay close to the actual mean, incurring small regret.
\end{itemize} 

\subsection{Motivating Applications of Trajectory-wise Regret}
\emph{Clinical Trials:} The trajectory-level view is more appropriate when we have \emph {non-replayability} like in clinical trials (the motivating example of \cite{barman,sayak}). Here, each
patient is treated exactly once. The sequence of treatments administered is a single trajectory. A metric that averages over multiple instantiations of the algorithm’s randomization does not reflect what actually happens to the population of patients in any given run. In such cases, a high probability per-trajectory regret is of importance, which our framework provides (see high probability Nash regret, Theorem~\ref{thm:hp}).

\emph{Risk-Sensitive Control:} In classical risk-sensitive control (\cite{risk-howard,whittle1990risk} as well as the Linear-Exponential-Quadratic-Gaussian framework (\cite{biswas2023ergodic}) the objective is to minimize the expected cumulative cost given by $\E (\exp(-\theta C))$, were $\theta$ is the risk sensitive parameter and $C$ is the cost incurred in one trajectory. As discussed in \cite{biswas2023ergodic}, the expectation outside encodes risk sensitivity, and reducing it to $\exp(-\theta \E C)$ collapses into risk neutral setup.

\emph{Portfolio Optimization and Kelly Criteria: }
Consider a sequential investment problem where at each round $t$, a fraction of wealth is allocated
to asset $I_t$ yielding a (multiplicative) return $R_t$. After $T$ rounds, terminal wealth is $W_T = \prod_{t=1}^T R_t$. The Kelly criterion \cite{kelly} maximizes $\E W_T$  which is equivalent to trajectory-wise expected regret.

\emph{Ergodicity Economics:} 
 \cite{peters2019ergodicity} argues that for non-ergodic multiplicative processes, the time-average growth rate $\lim_{T\to\infty}(\prod_t X_t)^{1/T}$ differs from the ensemble average $\mathbb{E}[X_t]$, and rational agents should optimize the former.

\emph{Evolutionary Biology:} The geometric mean fitness principle (\cite{dempster1955maintenance,lewontin1969population}) states that long-run reproductive success in varying environments is governed by $\mathbb{E}[(\prod_t W_t)^{1/T}]$, not $(\prod_t \mathbb{E}[W_t])^{1/T}$, where $W_t$ is the fitness in generation $t$. Strategies producing occasional near-zero fitness are selected against even if their arithmetic mean fitness is high.

\section{Other Related Work}
\label{sec:related}
\emph{Fairness in Online Learning and Bandits:}
\cite{joseph2016fairness} initiates meritocratic fairness in multi-armed bandits, requiring
that a worse arm is never preferred over a better one, with tight regret bounds.
\cite{joseph2018fair} extends this to contextual bandits under a Rawlsian notion with
polynomial regret bounds. \cite{gillen2018online} studies individual fairness in linear
contextual bandits with an unknown Mahalanobis similarity metric, learning constraints from
weak regulatory feedback. \cite{wang2021fairness} formalizes exposure fairness in stochastic
bandits, designing algorithms that jointly minimize reward and fairness regret.

\emph{Nash Regret in Other Settings:}
Building on \cite{barman}, \cite{sawarni2023nash} extends Nash regret to stochastic linear
bandits with tight upper bounds via log-transformed reward maximization.
\cite{sarkar2026improved} resolves the open problem of dimension-optimal Nash regret in
linear bandits and initiates the study of $p$-mean regret, proposing \textsc{FairLinBandit}.
\cite{zhang2024noregret} extends the NSW objective to multi-agent bandits with sublinear
regret guarantees and matching lower bounds in both stochastic and adversarial settings.

\emph{Axiomatic Fairness in Mathematical Economics:}
\cite{nash1950bargaining} introduces the Nash bargaining solution via axioms of symmetry,
scale invariance, independence of irrelevant alternatives, and Pareto optimality.
\cite{kaneko} proves NSW is the \emph{unique} social welfare function satisfying classical
social choice axioms without interpersonal utility comparisons. \cite{moulin} formalizes the
Pigou-Dalton transfer principle, positioning NSW between utilitarian and egalitarian
objectives. \cite{unreasonable} shows maximum Nash welfare simultaneously achieves EF1 and
Pareto optimality. \cite{cousins2023fairpac} characterizes the $p$-mean welfare family in a
machine learning context, identifying NSW as its unique fair member.

\section{Problem Formulation}
\label{sec:setup}
We work with the canonical stochastic bandit model. where the sample space is $\Omega = [k]^T \times [0,1]^{k\times T}$ equipped with the product $\sigma$-algebra. An element $\omega = ((I_1, \ldots, I_T), (Y_{i,s})_{i \in [k], s \in [T]})$ specifies the algorithm's pull sequence and a $Y_{i,s}$ is the  reward of arm $i$ on its $s$-th pull. Throughout, we assume that rewards $Y_{i,s}$ are in $[0,1]$ for all $i\in [k]$ and $s\in [T]$. Moreover, the rewards obtained from each arm are i.i.d.\ with with mean $\mu_i$ for the $i$-th arm. Recall that we assume $\mu_i \in (0,1]$ for all $i \in [k]$.

Any learning algorithm $\mathcal{A}$ recommends arm $I_t$ for the learner at time $t$ based on past history. As mentioned in Section~\ref{sec:intro}, the objective is to characterize the \emph{trajectory-wise Nash regret} as well as the \emph{high probability Nash regret} given respectively by
\[
   \NRtilde_T \;=\; \mu^\star \;-\; \E\![(\prod_{t=1}^T \mu_{I_t})^{\!1/T}] \geq \mathrm{NR}_T, \quad \NRhat_T := \mu^\star - \bigl(\prod_t \mu_{I_t}\bigr)^{1/T}\!.
\]
First, we show that a positive floor assumption on $\mu_i$ is needed for tractability.

\begin{assumption}[Positivity floor]\label{ass:floor}
There exists a constant $\mu_{\min} \in (0, \mu^\star]$ such that $\mu_i \ge \mu_{\min}$ for every $i \in [k]$.
\end{assumption}
We now argue that such assumption is unavoidable for handling the trajectory-wise and high probability Nash regret. The following proposition formally shows that any algorithm $\mathcal{A}$ which pulls every arm with positive probability at least once in $T$ rounds, the modified Nash regret cannot go to zero, and hence learning is not possible.

\begin{proposition}\label{prop:floor-needed}
Fix $k \ge 2$. For any algorithm $\mathcal{A}$ that pulls every arm with positive probability at least once over $T$ rounds, there exists a bandit instance $\nu$ with $\mu^\star = 1$ and $\min_i \mu_i = 0$ on which
\[
   \NRtilde_T(\mathcal{A}, \nu) \;\ge\; \mu^\star \cdot \Pr \!\bigl[\exists\, t \le T:\, I_t = i_0\bigr],
\]
where $i_0$ is the zero-mean arm. In particular, no learning algorithm achieves $\NRtilde_T \to 0$ in $T$ uniformly over instances allowing zero-mean arms.
\end{proposition}
\begin{remark}[Assumption~\ref{ass:floor} for High Probability Regret]
    Without Assumption~\ref{ass:floor}, on the event that an algorithm pulls a zero-mean arm at least once with non-zero probability, $\NRhat_T = \mu^\star$ exactly. Hence, no high-probability bound below $\mu^\star$ is possible without the floor. 
\end{remark}
\begin{remark}[Such Assumption is not required for ensemble Nash regret of \cite{barman,sayak}]
    This pathology disappears for the ensemble $\NR_T$, where the expectation is taken \emph{before} the geometric mean. Hence, a single zero-mean round is averaged out into a $\Theta(1/T)$ contribution to $\E[\mu_{I_t}]$.
\end{remark}
\begin{remark}[No Impact on Regret (order-wise)]
We see (in Theorems~\ref{thm:exp}, \ref{thm:hp} that the floor enters our bound only logarithmically and only in the lower-order $1/T$ term.  Hence, such assumption (i.e., a low value of $\mu_{\min}$ does not impact the regret scaling. \end{remark}

\begin{remark}[Knowledge of $\mu_{\min}$ not needed]
We see in Section~\ref{sec:algo}, that the knowledge of $\mu_{\min}$ is not required to run the learning algorithm. This is just for theoretical tractability.
\end{remark}

\section{Algorithm}
\label{sec:algo}
We propose the Round Robbin Nash Confidence Bound algorithm (\texttt{RR-NCB}), parameterized by a confidence parameter $L > 0$ in place of $\log T$. Set $c = 3$ and $S := c^2 L /\mu^\star$ throughout.

We use the following notation throughout. The algorithm observes $Y_{I_t, n_{I_t}(t)}$ at round $t$, where $n_i(t) := \sum_{s=1}^t \Ind\{I_s = i\}$. We write $\widehat\mu_{i,s} := \frac{1}{s}\sum_{j=1}^s Y_{i,j}$ for the empirical mean of arm $i$ from its first $s$ rewards, and abbreviate $n_i=n_i(t)$, $\widehat\mu_i = \widehat\mu_{i, n_i(t)}$ when $t$ is clear from context.

\begin{algorithm}[t!]
\label{algo:main}
\caption{Round Robbin Nash Confidence Bound (\texttt{RR-NCB})}\label{alg:ncb}
\begin{algorithmic}[1]
\Require Number of arms $k$, horizon $T$, confidence parameter $L > 0$; constants $c=3$.
\State \textbf{Phase 1 (Round-robin exploration).}
\State $t \gets 0$
\While {$\max_{i \in [k]} \, n_i \cdot \widehat{\mu}_i \;\leq \; 420\, c^2 L$}
   \State Pull each arm $i \in [k]$ in round robbin fashion.
   \State Update $n_i$ and $\widehat\mu_i$.
   \State $t \gets t + 1$
\EndWhile
\State Let $\tau$ denote the round at which the stopping condition first holds.
\State \textbf{Phase 2 (NCB selection).}
\For{$t = \tau+1, \ldots, T$}
   \State Set $\NCB_i(t) := \widehat\mu_i + 2c\sqrt{2\widehat\mu_i L /n_i}$ for each arm $i$.
   \State Pull $I_t \in \argmax_i \NCB_i(t)$
   \State Observe $X_t$; update $n_{I_t}, \widehat\mu_{I_t}$.
\EndFor
\end{algorithmic}
\end{algorithm}

We will instantiate Algorithm~\ref{alg:ncb} twice: (a) with $L = \log T$, for the in-expectation bound on $\NRtilde_T$, and (b)
with $L = L_\delta := \log(8kT/\delta)$, for the high-probability bound on $\NRhat_T$.

The algorithm is a two-phase procedure for the $k$-armed stochastic bandit problem over a known horizon $T$, parameterized by a confidence parameter $L > 0$ and a constant $c = 3$. We modify the Nash Confidence Bound (NCB) algorithm of \cite{barman}, adapted to support both in-expectation and high-probability Nash regret guarantees through a single unified design. 
 
\paragraph{Phase 1: Round-robin exploration.} The algorithm begins by exploring all arms uniformly. In each pass, it pulls every arm $i \in \{1,\ldots,k\}$ exactly once in a round robbin fashion, then checks whether the stopping condition 
\begin{align*}
    \max_{i} n_i \cdot \widehat\mu_i > 420\, c^2 L
\end{align*}
is satisfied, where $n_i$ is the total number of times arm $i$ has been pulled so far and $\widehat\mu_i$ is its empirical mean. If the condition is not met and the horizon has not been reached, another pass begins. This round-robin schedule ensures that all arms have been pulled the same number of times. The stopping time $\tau$ is the first round at which the condition is met.
 
The stopping condition serves a dual purpose. First, it guarantees that Phase~1 does not end too early: the threshold $420\,c^2 L$ is large enough that, on the good event, no arm can cross it before round $192\,kS$ (where $S := c^2 L/\mu^\star$), ensuring every arm accumulates at least $192\,S$ samples. This sample-count floor is what makes Phase~2's concentration-based analysis viable. Second, the condition guarantees that Phase~1 does not run too long: the optimal arm's product $n_{i^\star}\widehat\mu_{i^\star}$ must cross the threshold by round $484\,kS$, keeping Phase~1's contribution to the regret bounded by $O(kL\log(1/\mu_{\min})/T)$.
 
\paragraph{Phase 2: NCB selection.} From round $\tau + 1$ to $T$, the algorithm selects arms greedily by an optimistic index. For each arm $i$, it computes the Nash Confidence Bound (NCB) given by 
\begin{align*}
    \NCB_i := \widehat\mu_i + 2c\sqrt{2\,\widehat\mu_i\, L/n_i}
\end{align*}
and pulls the arm with the largest $\NCB_i$. The index scales with $\sqrt{\widehat\mu_i/n_i}$ rather than $\sqrt{1/n_i}$ similar to \cite{barman}. This ensures that arms whose empirical means are small automatically receive tighter confidence widths, preventing over-exploration. This variance-sensitive design is the key to achieving $O(\sqrt{k/T})$ Nash regret rather than $O(\sqrt{k\log k/T})$. After each pull, the selected arm's count and empirical mean are updated, and the procedure repeats until round~$T$.
 
\section{Main Results}
\label{sec:results}
In this section, we obtain the main result of the paper. Specifically, we show that \texttt{RR-NCB} yields trajectory-wise as well as high probability Nash regret guarantee. We formally state the results now.

\begin{theorem}[Trajectory-wise Expected Regret]\label{thm:exp}
Suppose Assumption~\ref{ass:floor} holds and we run  Algorithm~\ref{alg:ncb} with $L=\log T$, $c = 3$ and horizon $T$. The Trajectory-wise expected regret $\NRtilde_T$ satisfies
\[
  \NRtilde_T \;\le\; C_1 \sqrt{\tfrac{k\log T}{T}} + C_2 \tfrac{k\log T \log(1/\mu_{\min})}{T} + \tfrac{3k\log(1/\mu_{\min})}{T^2}
\]
for absolute constants $C_1, C_2 > 0$.
\end{theorem}
The proof is deferred to the Appendix. A few remarks are in order.
\begin{remark} (Matching regret with \cite{barman,sayak}
    The regret $\NRtilde_T$ consists of three terms, out of which the dominating term is $\tildeO(\sqrt{k/T})$. The regret scaling matches with the (ensemble) Nash regret of \cite{barman,sayak}. This shows that although we are working with a stronger notion of regret (as $\NRtilde_T \geq \mathrm{NR}_T$), we obtain the same result (order-wise). In other words, trajectory-wise or sample path based analysis comes with no additional cost (in terms of regret).  
\end{remark}
\begin{remark} (Matching Regret Lower Bound)
    Since $\NRtilde_T \geq \mathrm{NR}_T$, a simple application of AM-GM inequality shows that $\NRtilde_T = \Omega(\sqrt{k/T})$, which implies that \texttt{RR-NCB} is order-wise optimal.
\end{remark}
\begin{remark} (Dependence on $\mu_{\min}$)
    We see that although $\NRtilde_T$ depends on the positivity floor $\mu_{\min}$, the dependence is present only in minor terms (i.e., the terms scaling with $1/T$ and $1/T^2$). Hence, the influence of $\mu_{\min}$ on regret is minor. In Section~\ref{sec:sim}, we validate this finding through simulations. 
\end{remark}

We now present the high probability Nash regret. We have the following result.
\begin{theorem}[High Probability Regret]\label{thm:hp}
Let $\delta \in (0, 1)$ and we run Algorithm~\ref{alg:ncb} with $L = \log(8kT/\delta)$. Under Assumption~\ref{ass:floor}, we have
\[
  \Pr\!\left[\,\NRhat_T \;\le\; C' \sqrt{\tfrac{k \log(kT/\delta)}{T}} \;+\; C'' \,\tfrac{k\,\log(kT/\delta)\,\log(1/\mu_{\min})}{T}\,\right] \;\ge\; 1 - \delta,
\]
for absolute constants $C', C'' > 0$.
\end{theorem}
The formal proof is deferred to the Appendix. We collect a few remarks here.
\begin{remark} (Matching regret upper and lower bounds)
Similar to above, the high probability regret is $\tildeO(\sqrt{k/T})$ which matches \cite{barman,sayak} and is optimal (via same AM-GM inequality).    
\end{remark}
\begin{remark}(Novel high probability Nash regret)
    We point out that this is the first result obtaining high probability Nash regret, which is of importance in safety critical applications.
\end{remark}
\section{Proof Sketch}
\label{sec:sketch}
Recall that we have $k$-armed stochastic bandit with means $\mu_i \in [\mu_{\min}, 1]$, $\mu^\star = \max_i \mu_i$, horizon $T$. We bound two variants of Nash regret:
\[
   \NRtilde_T := \mu^\star - \E\!\bigl[(\textstyle\prod_t\mu_{I_t})^{1/T}\bigr], \qquad
   \NRhat_T := \mu^\star - (\textstyle\prod_t\mu_{I_t})^{1/T}.
\]
Our proposed algorithm, \texttt{RR-NCB} (given in \ref{alg:ncb}) has 2 phases. In Phase I agent pull arms in a round-robin fashion until $\max_i n_i\widehat\mu_i > 420\,c^2L$ ($c=3$). In Phase II, we play the arm with maximum index given by $\NCB_i = \widehat\mu_i + 2c\sqrt{2\widehat\mu_i L/n_i}$ in Phase~2. We set $L = \log T$ for $\NRtilde_T$, $L = L_\delta = \log(8kT/\delta)$ for $\NRhat_T$, and $S = c^2 L/\mu^\star$. We now provide step by step reduction.
 
\medskip\noindent\textit{Step 1: Reduction to Trajectory-wise log-regret:} The first step in the proof is to reduce the variants of Nash regret to log-regret trajectory-wise.  Since $(\prod_t\mu_{I_t})^{1/T} = \exp\bigl(\log\mu^\star - \widehat R_T\bigr)$ trajectory-wise, where $\widehat R_T := \tfrac{1}{T}\sum_t\log(\mu^\star/\mu_{I_t}) \ge 0$, the inequality $1 - e^{-x} \le x$ gives \emph{trajectory-wise} $\NRhat_T \le \mu^\star\widehat R_T$. For the in-expectation case, we apply Jensen inequality on the exponential convex function, which gives $\NRtilde_T \le \mu^\star\bar R_T$, where $\bar R_T = \E[\widehat R_T]$.

\medskip\noindent\textit{Step 2: Define a Good Event $E$:} We define $E := E_2 \cap E_3$ that supports the entire analysis. The definitions are parameterized by $L > 0$, with $S := c^2 L /\mu^\star$ with $c=3$.

\begin{description}
\item[$E_2$ (concentration of high-mean arms):] For all arms $i$ with $\mu_i > \mu^\star/64$ and all sample counts $s \in [64S, T], \quad |\widehat\mu_{i,s} - \mu_i| \;\le\; c\sqrt{\tfrac{\mu_i L}{s}}$.
\item[$E_3$ (low-mean arms stay low):] For all arms $j$ with $\mu_j \le \mu^\star/64$ and all sample counts $s \in [64S, T], \quad \widehat\mu_{j,s} \;<\; \tfrac{\mu^\star}{32}$.
\end{description}

We show that on  good event $E$, we have 
\begin{align*}
    \Pr[E] \ge 1-3k/T^2 \quad \text{for} \quad  L=\log T, \quad \Pr[E]\ge 1-\delta \quad \text{for} \quad L = L_\delta
\end{align*}
\medskip\noindent\textit{Step 3: Regret split into Phase 1 and Phase 2}:
\begin{itemize}
\item \emph{Phase 1} ($t \le \tau$): each $\log(\mu^\star/\mu_{I_t}) \le \log(1/\mu_{\min})$, so the regret contribution is at most $\mu^\star \cdot \tau \log(1/\mu_{\min})/T$.
\item \emph{Phase 2} ($t > \tau$): We show (in Lemma~\ref{lem:lem11} that on event $E$, bad arms are never played in Phase 2. In particular every arm $j$ with $\mu_j \le \mu^\star/64$ is never pulled in Phase~2. The condition on pulled arm $\mu_{I_t} \ge \mu^\star/64$  lets us linearize, $\log(\mu^\star/\mu_i) \le 64(\mu^\star-\mu_i)/\mu^\star$, reducing Phase 2's log-regret to $(64/T)\sum_{t>\tau}(\mu^\star - \mu_{I_t})$. This is the standard cumulative gap and is handled via concentration arguments.
\end{itemize}
 
\medskip\noindent\textit{Step 4: Bound on Phase 1 Termination:} Two auxiliary lemmas pin down the duration of Phase I:
\begin{itemize}
\item Lemma~\ref{lem:lem7}: At $\tau \le 192 kS$, $n\widehat\mu_{i,n} < 420\,c^2 L$ (proved by splitting on $\mu_i \lessgtr \mu^\star/64$ and using $E_2,E_3$). Hence, Phase-1 hasn't stopped yet.
\item Lemma~\ref{lem:lem8}: At $\tau \ge 484 kS$, $n\widehat\mu_{i^\star,n} > 420\,c^2 L$ (using $E_2$ at the optimal arm). So, Phase 1 has definitely stopped.
\end{itemize}
% Round-robin makes counts equal across arms at every pass-end, so at round $128 kS$ every arm has count $\le 128 S < 192 S$ and Lemma 7 gives $n_i\widehat\mu_i < 210 c^2 L < 420 c^2 L$ — \emph{not yet stopped}. At round $\lceil 968 S \rceil k$, the optimal arm has count $\ge 968 S \ge 484 S$ and Lemma 8 gives $n_{i^\star}\widehat\mu_{i^\star} \ge 462 c^2 L > 420 c^2 L$ — \emph{stopped}. Hence $128 kS < \tau \le 970 kS$, and every arm enters Phase 2 with count $> 128 S$.
 
\medskip\noindent\textit{Step 5: $\sqrt{kT}$ NCB analysis.} We characterize the per-pull gap in Lemma~\ref{lem:lem12}. In particular, 
\begin{align*}
    \mu^\star - \mu_i \le 4c\sqrt{\mu^\star L/(T_i-1)}
\end{align*}
for any arm pulled in Phase 2. This is obtained by exploiting Lemma~\ref{lem:lem10} ($\NCB_{i^\star} \ge \mu^\star$) along with the event $E_2$. Using this, we show that the $j$-th pull of arm $i$ in Phase 2 yields a maximum regret of $4c\sqrt{\mu^\star L/j}$. Summing over all the pulls and applying the Cauchy-Schwartz inequality, we obtain
\begin{align*}
    \sum_{t>\tau}(\mu^\star - \mu_{I_t}) \le 8c\sqrt{kT\mu^\star L}.
\end{align*}
\medskip\noindent\textit{Step 5: Combine.} Phase 1 contributes $\le C_1 kL\log(1/\mu_{\min})/T$, Phase 2 contributes $\le C_2\sqrt{kL/T}$. Multiplying by $\mu^\star \le 1$ and adding:
\[
   \NRtilde_T,\ \NRhat_T \;=\; O\!\Bigl(\sqrt{\tfrac{kL}{T}} \;+\; \tfrac{kL\log(1/\mu_{\min})}{T}\Bigr).
\]
For the Trajectory-wise expected regret, we also calculate the regret off event $E$. By using the high probability of $E$, this term yields a minor $\mathcal{O}(1/T^2)$ term.
\section{Simulations}
\label{sec:sim}

\begin{figure}[t!]
    \centering
    \includegraphics[width=0.32\textwidth]{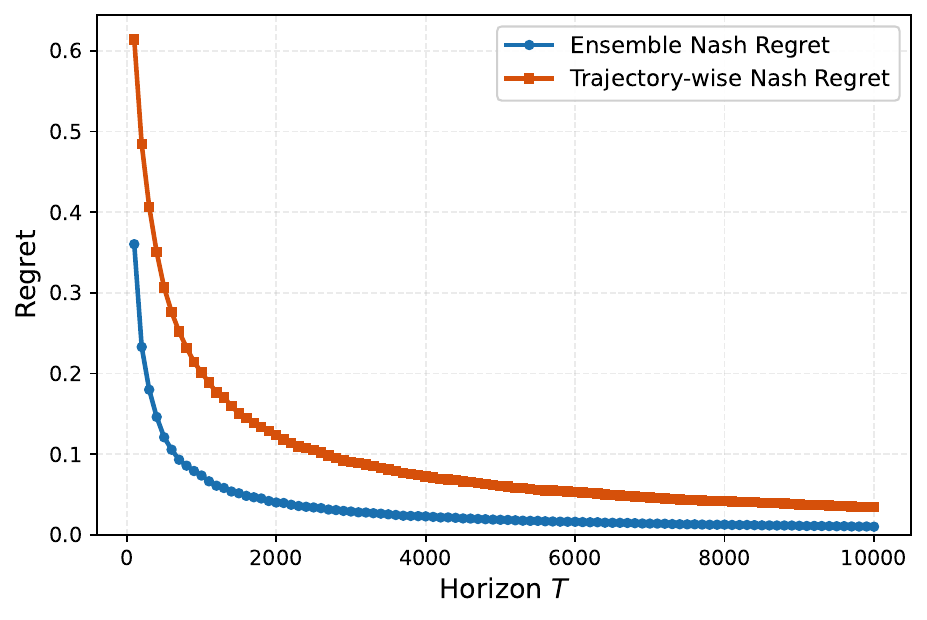}
    \hfill
    \includegraphics[width=0.32\textwidth]{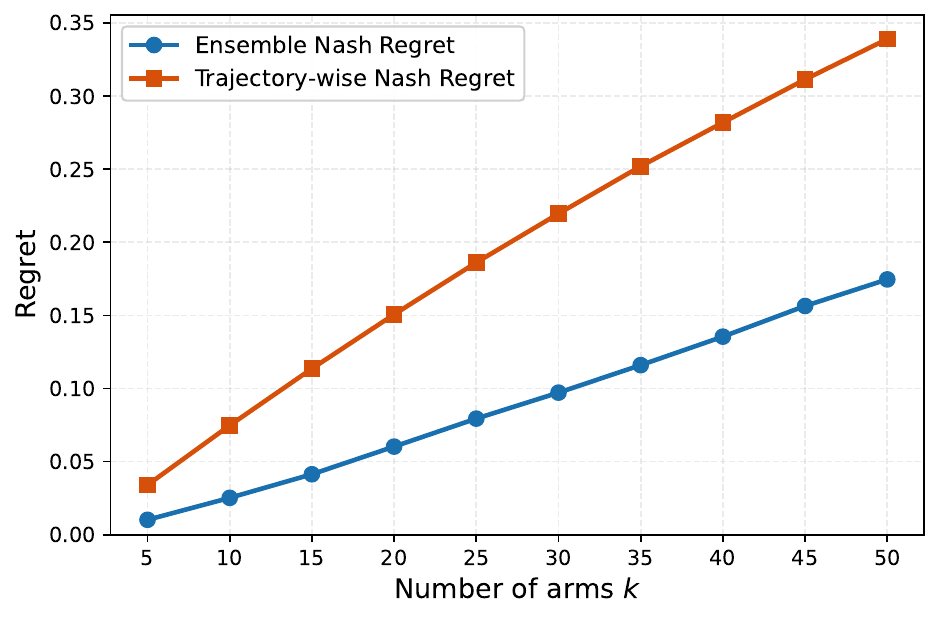}
    \hfill
    \includegraphics[width=0.32\textwidth]{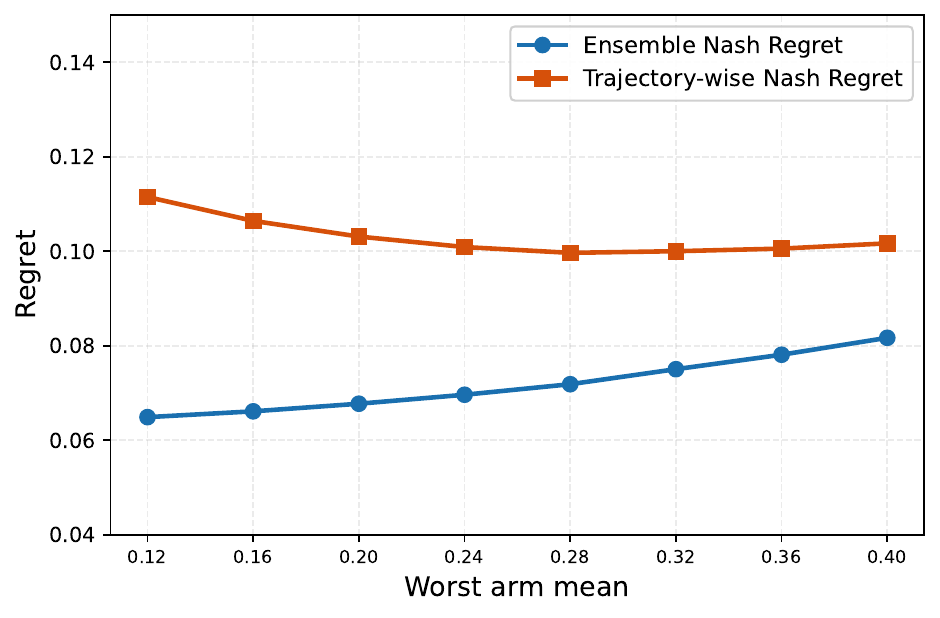}
    \caption{Nash Regret comparisons. \textit{Left:} Regret vs.\ horizon $T$. \textit{Center:} Regret vs.\ number of arms $k$. \textit{Right:} Regret vs.\ worst arm mean.}
    \label{fig:nash_regret}
\end{figure}

We conduct three sets of simulations on a
stochastic multi-armed bandit instance with the best arm mean $\mu^{\star} = 0.9$ and worst arm mean $\varepsilon$ (we vary this). We evaluate two notions of Nash regret across 200 independent seeds: the
\emph{ensemble Nash regret} (of \cite{barman}), computed by averaging pulled-arm means across seeds at each round
before taking the geometric mean over rounds, and the \emph{trajectory-wise Nash regret}, computed by taking the geometric mean of realized rewards along each individual trajectory and
then averaging across seeds.

In the first experiment we fix $k = 25$ arms and $\varepsilon = 0.1$, and vary the horizon
$T \in \{100, 200, \ldots, 10{,}000\}$.
In the second experiment we fix $T = 10{,}000$ and $\varepsilon = 0.1$, and vary the number
of arms $k \in \{5, 10, 15, \ldots, 50\}$.
In the third experiment we fix $T = 10{,}000$ and $k = 25$, and vary the worst-arm mean
$\varepsilon$ across 10 values from $0.12$ to $0.40$ in increments of $0.04$.
 
% UCB1 pulls each arm once in the first $k$ rounds and subsequently selects the arm maximizing
% $\hat{\mu}_i + \sqrt{2 \log t\,/\,n_i}$, where $\hat{\mu}_i$ and $n_i$ are the empirical mean
% and pull count of arm $i$ at round $t$.
% This algorithm is well-suited for demonstrating the distinction between the two regret notions
% because it is adaptive---it rapidly concentrates pulls on the best arm---so both the horizon and
% the number of arms meaningfully affect the pull distribution, and the impact of bad-arm pulls on
% trajectory geometric means remains visible throughout the learning process.
 
\subsection{Results and Insights}
The results are shown in Figure~\ref{fig:nash_regret}. In the regret vs.\ $T$ experiment, we compare the performance of \texttt{RR-NCB} to  \texttt{NCB} of \cite{barman}. We see that both regrets go to zero implying learnability and sub-linear regret. Moreover, as theory predicts, the trajectory-wise Nash regret is bigger than the standard ensemble Nash regret. However as $T$ gets large, the gap between the two reduces. This validates the theoretical finding that both the algorithms achieve similar order-wise regret.

We then compare the Nash regret against the number of arms $k$, and unsurprisingly, we got an increasing curve. To understand the impact of the positivity error floor assumption in our formulation, we run another experiment where we choose $10$ bandit instances with varying arm mean for the worst arm. We take the mean of the worst arm as a proxy for positivity error floor (see Assumption~\ref{ass:floor}). We observe (in Figure~\ref{fig:nash_regret}) that the trajectory-wise Nash regret is not changing much from one instance to other. This observation further validates the theoretical finding that the effect of this minimum floor assumption is in the minor order regret terms, and poses no major influence on the trajectory-wise Nash regret. We run \texttt{NCB} of \cite{barman} in the same setting as well. As expected the ensemble Nash regret is not changing much as well, which substantiate the fact that the ensemble Nash regret does not depend on such positivity floor.

\section{Conclusion and Open Problems}
In this paper we study fair multi-armed bandits from a trajectory (sample path) point of view (both in expectation and high probability). We propose and analyze a learning algorithm which yields optimal regret guarantees which we validate through simulations.

We end this paper with a few open problems. Until now, all the guarantees on Nash regret are instance independent. Can we obtain instance dependent optimal logarithmic Nash regret? Moreover, how does the trajectory-wise Nash regret scale in the presence of multiple agents (in a cooperative or competitive setup)? Furthermore, for structured bandits (like linear, kernelized), can we obtain regret guarantees for trajectory-wise Nash regret? We keep these questions for our future endeavors.

\bibliography{ref}
\bibliographystyle{plainnat}

\appendix
\newpage
\section{Proofs}
\label{sec:proofs}
\subsection{Setup and Useful Lemmas}
Without loss of generality, throughout the proof we assume
\begin{align*}
    \mu^\star \ge 512\sqrt{kL/T},
\end{align*}
which we call the warm-up condition. When this bound fails, our results hold deterministically with 
\[
\NRtilde_T \le \mu^\star \le 512\sqrt{kL/T}.
\]

Similar bounds holds for $\NRhat_T$. This crucially leverages the non-negativity of the rewards.

\subsection{The Good Event}\label{sec:goodevent}

We define a good event $E := E_2 \cap E_3$ that supports the entire analysis. The definitions are parameterized by $L > 0$, with $S := c^2 L /\mu^\star$ as before. Recall $c=3$.

\begin{description}
\item[$E_2$ (concentration of high-mean arms):] For all arms $i$ with $\mu_i > \mu^\star/64$ and all sample counts $s \in [64S, T]$,
\[
   |\widehat\mu_{i,s} - \mu_i| \;\le\; c\sqrt{\tfrac{\mu_i L}{s}}.
\]
\item[$E_3$ (low-mean arms stay low):] For all arms $j$ with $\mu_j \le \mu^\star/64$ and all sample counts $s \in [64S, T]$,
\[
   \widehat\mu_{j,s} \;<\; \tfrac{\mu^\star}{32}.
\]
\end{description}

\subsection{Round-robin sampling}

Under the round-robin schedule of Algorithm~\ref{alg:ncb}, after any number of rounds $r \ge k$ during Phase~1, each arm $i$ has been pulled either $\lfloor r/k\rfloor$ or $\lceil r/k\rceil$ times. For the sake of simplicity and clarity of exposition we assume that $\tau$ is a multiple of $k$ (otherwise we just use ceil and floor which will only impact the final regret in terms of universal constants). Hence, at round $\tau$ every arm has count $\tau/k$ exactly. We use this throughout below.

\subsection{Probability of the good event}

\begin{lemma}[Good event probability]\label{lem:goodevent}
$\Pr[E^c] \le 3kT\exp(-c^2 L/3)$.
\end{lemma}

\begin{proof}
We bound $\Pr[E_2^c]$ and $\Pr[E_3^c]$ separately and union-bound.

\emph{$E_2$:} For an arm $i$ with $\mu_i > \mu^\star/64$ and $s \in [64S, T]$, by Lemma~\ref{lem:mult-chernoff} with $\nu = \mu_i$ and $\delta := c\sqrt{L/(\mu_i s)}$, we have
\[
   \Pr\!\left[|\widehat\mu_{i,s} - \mu_i| \ge c\sqrt{\tfrac{\mu_i L}{s}}\right] \le 2\exp\!\Bigl(-\tfrac{\delta^2 s\mu_i}{3}\Bigr) = 2\exp\!\Bigl(-\tfrac{c^2 L}{3}\Bigr).
\]
We require $\delta \le 1$ for Lemma~\ref{lem:mult-chernoff} to apply, i.e., $s \ge c^2 L/\mu_i$. Since $\mu_i > \mu^\star/64$, $c^2 L/\mu_i < 64 c^2 L/\mu^\star = 64S$, and $s \ge 64S$ suffices.

Union-bounding over $k$ arms and $T$ sample counts: $\Pr[E_2^c] \le 2kT \exp(-c^2 L/3)$.

\emph{$E_3$:} For an arm $j$ with $\mu_j \le \mu^\star/64$ and $s \in [64S, T]$, by Lemma~\ref{lem:hoeff-ext} with $\nu_H = \mu^\star/64$ (note $(1+1)\nu_H = \mu^\star/32$), we have
\[
   \Pr\!\left[\widehat\mu_{j,s} \ge \tfrac{\mu^\star}{32}\right] \le \exp\!\Bigl(-\tfrac{1 \cdot s \cdot \mu^\star/64}{3}\Bigr) = \exp\!\Bigl(-\tfrac{s \mu^\star}{192}\Bigr).
\]
For $s \ge 64S$, we have $s\mu^\star/192 \ge 64 c^2 L/192 = c^2 L/3$. So the bound becomes $\exp(-c^2 L/3)$.

Union over $k$ arms and $T$ sample counts: $\Pr[E_3^c] \le kT \exp(-c^2 L/3)$.

Adding the two bounds: $\Pr[E^c] \le 3kT\exp(-c^2 L/3)$.
\end{proof}

\begin{corollary}\label{cor:goodevent-instances}
For $c = 3$:
\begin{enumerate}
\item With $L = \log T$ and $T \ge 3k$, $\Pr[E] \ge 1 - 3k/T^2$.
\item With $L = L_\delta := \log(8kT/\delta)$ and $\delta \in (0,1)$, $\Pr[E] \ge 1 - \delta$.
\end{enumerate}
\end{corollary}

\begin{proof}
\emph{(1)} With $L = \log T$ and $c=3$: $\exp(-c^2 L/3) = \exp(-3\log T) = T^{-3}$. So $\Pr[E^c] \le 3k/T^2$.

\emph{(2)} With $L = \log(8kT/\delta)$: $\exp(-c^2 L/3) = (8kT/\delta)^{-3}$. So
\[
   \Pr[E^c] \le \tfrac{3kT \delta^3}{(8kT)^3} = \tfrac{3 \delta^3}{512 (kT)^2} \le \delta,
\]
for any $\delta \in (0,1)$ and $kT \ge 1$.
\end{proof}

\subsection{Analysis of Phase~1}\label{sec:phase1}

We now characterize the number of pulls made in Phase-I. The Phase~1 stopping rule activates when $\max_i n_i \widehat\mu_i > 420 c^2 L$. To pin down $\tau$, we need to know when $n_i \widehat\mu_i$ has not yet crossed the threshold and when it must have. The next two lemmas provide both bounds.

\begin{lemma}
\label{lem:lem7}
On the event $E$, for any arm $i \in [k]$ and any sample count $n \le 192 S$,
\[
   n\,\widehat\mu_{i,n} \;<\; 210\, c^2 L < 420c^2L.
\]
\end{lemma}

\begin{proof}
Since $Y_{i,j} \in [0,1]$ (non-negative), the sum $n\widehat\mu_{i,n} = \sum_{j=1}^n Y_{i,j}$ is non-decreasing in $n$. So it suffices to bound the case $n = N$ where $N = 192 S$. We have
\begin{equation}\label{eq:lem7-mono}
   n\widehat\mu_{i,n} \;\le\; N\widehat\mu_{i,N}.
\end{equation}
Two cases on $\mu_i$.

\emph{Case 1: $\mu_i \le \mu^\star/64$.} Since $N = 192S \ge 64S$, $E_3$ applies and gives $\widehat\mu_{i,N} < \mu^\star/32$. Hence
\[
N\widehat\mu_{i,N} \;<\; 192 S \cdot \tfrac{\mu^\star}{32} \;=\; 6 S \mu^\star \;=\; 6 c^2 L.
\]

\emph{Case 2: $\mu_i > \mu^\star/64$.} Since $N \ge 64 S$, $E_2$ applies:
\[
\widehat\mu_{i,N} \;\le\; \mu_i + c\sqrt{\tfrac{\mu_i L}{N}} \;\le\; \mu^\star + c\sqrt{\tfrac{\mu^\star L}{N}}
\;=\; \mu^\star + c\sqrt{\tfrac{\mu^\star L \cdot \mu^\star}{192 c^2 L}} \;=\; \mu^\star\!\left(1 + \tfrac{1}{\sqrt{192}}\right)\!.
\]
Hence, invoking the definition of $N$, we  have $N\widehat\mu_{i,N} < 210 c^2 L$.

In both cases $N\widehat\mu_{i,N} < 210 c^2 L$, so~\eqref{eq:lem7-mono} gives $n\widehat\mu_{i,n} < 210 c^2 L$.
\end{proof}

\begin{lemma}\label{lem:lem8}
On the event $E$, for any sample count $n \ge 484 S$,
\[
   n\,\widehat\mu_{i^\star, n} \;\ge\; 462\, c^2 L > 420 c^2L.
\]
\end{lemma}

\begin{proof}
By monotonicity, $n\widehat\mu_{i^\star,n} \ge M\widehat\mu_{i^\star,M}$ where $M = 484 S$. Since $\mu_{i^\star} = \mu^\star > \mu^\star/64$ and $M \ge 64S$, $E_2$ applies and we have
\[
\widehat\mu_{i^\star,M} \;\ge\; \mu^\star - c\sqrt{\tfrac{\mu^\star L}{M}} \;=\; \mu^\star - c\sqrt{\tfrac{\mu^\star L \cdot \mu^\star}{484 c^2 L}} \;=\; \mu^\star\!\left(1 - \tfrac{1}{\sqrt{484}}\right) \;=\; \tfrac{21}{22}\mu^\star.
\]
Hence $M\widehat\mu_{i^\star,M} \ge 462 c^2 L$.
\end{proof}

\subsubsection{Bounding the Phase~1 stopping time}
We use the above 2 Lemmas to bound $\tau$ under event $E$. We have
\begin{align*}
    192kS \leq \tau \leq 484kS (=T_e).
\end{align*}

\begin{corollary}[Phase~1 sample counts]\label{cor:phase1samples}
On $E$, at the end of Phase~1, every arm $i$ has been pulled exactly $\tau/k$ times, with $\tau/k > 192 S > 64 S$. In particular, every arm satisfies $n_i(\tau) > 192 c^2 L/\mu^\star$.
\end{corollary}

% \begin{proof}
% By Lemma~\ref{lem:phase1len}, $\tau > 128 kS$. Since $\tau$ is a multiple of $k$ (the algorithm only checks the stopping condition at the end of full round-robin passes), and round-robin sampling makes counts equal across arms, $n_i^{(1)} = \tau/k > 128 S$ for every $i$.
% \end{proof}

\subsection{Analysis of Phase~2: NCB-Based Selection}\label{sec:phase2}

Throughout this section we work on the good event $E$, and we denote $S = c^2 L/\mu^\star$ as before. By Corollary~\ref{cor:phase1samples}, on $E$ every arm $i$ has $n_i > 192 S > 64 S$ before the start of Phase~2 (the $\ge 64S$ form is what we use most often as it is defined in the event $E$).

\begin{lemma}[Optimal arm's NCB stays above $\mu^\star$]\label{lem:lem10}
On the event $E$ and assuming $\mu^\star \ge 512\sqrt{kL/T}$, for every Phase~2 round $t > \tau$,
\[
   \NCB_{i^\star}(t) \;\ge\; \mu^\star.
\]
\end{lemma}

\begin{proof}
Fix $t > \tau$ and let $n^\star := n_{i^\star}(t-1) \ge 64 S$ (Corollary~\ref{cor:phase1samples}). 

Since $\mu^\star > \mu^\star/64$ and $n^\star \ge 64 S$, $E_2$  gives $|\widehat\mu_{i^\star} - \mu^\star| \le c\sqrt{\mu^\star L/n^\star}$. In particular,
\begin{equation}\label{eq:opt-empirical}
\widehat\mu_{i^\star} \;\ge\; \mu^\star - c\sqrt{\tfrac{\mu^\star L}{n^\star}} \implies c\sqrt{\tfrac{\mu^\star L}{n^\star}} \geq \mu^\star - \widehat\mu_{i^\star}.
\end{equation}
We want to show $\widehat\mu_{i^\star} + 2c\sqrt{2\widehat\mu_{i^\star} L/n^\star} \ge \mu^\star$, equivalently
\begin{equation}\label{eq:opt-ncb-target}
2c\sqrt{\tfrac{2\widehat\mu_{i^\star} L}{n^\star}} \;\ge\; \mu^\star - \widehat\mu_{i^\star}.
\end{equation}
So~\eqref{eq:opt-ncb-target} is implied by
\begin{equation}\label{eq:opt-ncb-suff}
2c\sqrt{\tfrac{2\widehat\mu_{i^\star} L}{n^\star}} \;\ge\; c\sqrt{\tfrac{\mu^\star L}{n^\star}} \quad\Longleftrightarrow\quad 8 \widehat\mu_{i^\star} \ge \mu^\star.
\end{equation}
We verify $8\widehat\mu_{i^\star} \ge \mu^\star$. Using~\eqref{eq:opt-empirical} and $n^\star \ge 64 S = 64 c^2 L/\mu^\star$:
\[
   c\sqrt{\tfrac{\mu^\star L}{n^\star}} \;\le\; c\sqrt{\tfrac{\mu^\star L \cdot \mu^\star}{64 c^2 L}} \;=\; \tfrac{\mu^\star}{8}.
\]
Hence $\widehat\mu_{i^\star} \ge \mu^\star - \mu^\star/8 = 7\mu^\star/8$, so $8\widehat\mu_{i^\star} \ge 7\mu^\star \ge \mu^\star$. This establishes~\eqref{eq:opt-ncb-suff} and hence $\NCB_{i^\star}(t) \ge \mu^\star$.
\end{proof}

\begin{lemma}[Bad arms are never pulled in Phase~2]\label{lem:lem11}
On the event $E$ and assuming $\mu^\star \ge 512\sqrt{kL/T}$, every arm $j$ with $\mu_j \le \mu^\star/64$ is never pulled in Phase~2.
\end{lemma}

\begin{proof}
Fix such an arm $j$. By Corollary~\ref{cor:phase1samples}, $n_j \ge 64 S$ at the start of Phase~2. Since $\mu_j \le \mu^\star/64$, $E_3$ gives $\widehat\mu_j < \mu^\star/32$ throughout Phase~2 (note that $n_j$ only increases, and $E_3$ is a uniform statement over $s \in [64S, T]$).

For any Phase~2 round $t$, with $n_j \ge 64 S$:
\[
\NCB_j(t) \;=\; \widehat\mu_j + 2c\sqrt{\tfrac{2\widehat\mu_j L}{n_j}}
\;<\; \tfrac{\mu^\star}{32} + 2c\sqrt{\tfrac{2 (\mu^\star/32) L}{64 S}}
\;=\; \tfrac{\mu^\star}{32} + 2c\sqrt{\tfrac{\mu^\star L \cdot \mu^\star}{1024\, c^2 L}}
\;=\; \tfrac{\mu^\star}{32} + \tfrac{\mu^\star}{16} \;=\; \tfrac{3\mu^\star}{32}.
\]
By Lemma~\ref{lem:lem10}, $\NCB_{i^\star}(t) \ge \mu^\star$. Since $3\mu^\star/32 < \mu^\star \le \NCB_{i^\star}(t)$, we have $\NCB_j(t) < \NCB_{i^\star}(t)$, so arm $j$ is not selected at round $t$. This holds for every $t > \tau$, so $j$ is never pulled in Phase~2.
\end{proof}

\subsubsection{Per-pull gap for arms pulled in Phase~2}

\begin{lemma}[Per-pull gap]\label{lem:lem12}
On the event $E$ and assuming $\mu^\star \ge 512\sqrt{kL/T}$, suppose arm $i$ is pulled in Phase~2 at some round, and let $T_i$ denote its total pull count at that moment (including the current pull). Then
\[
   \mu^\star - \mu_i \;\le\; 4c\sqrt{\tfrac{\mu^\star L}{T_i - 1}}.
\]
\end{lemma}

\begin{proof}
By Lemma~\ref{lem:lem11}, $\mu_i > \mu^\star/64$. By Corollary~\ref{cor:phase1samples}, $T_i - 1 \ge 64 S$ (since the prior pull count is at least the Phase~1 count, which is $\ge 64S$). Hence $E_2$ applies at $s = T_i - 1$.

At the moment of this pull, since arm $i$ was selected by NCB,
\begin{equation}\label{eq:gap-1}
\widehat\mu_i + 2c\sqrt{\tfrac{2\widehat\mu_i L}{T_i - 1}} \;=\; \NCB_i \;\ge\; \NCB_{i^\star} \;\ge\; \mu^\star,
\end{equation}
the last inequality by Lemma~\ref{lem:lem10}.

We upper-bound $\widehat\mu_i$ using $E_2$ (since $\mu_i > \mu^\star/64$) and $T_i - 1 \ge 64 S$:
\[
\widehat\mu_i \;\le\; \mu_i + c\sqrt{\tfrac{\mu_i L}{T_i - 1}} \;\le\; \mu^\star + c\sqrt{\tfrac{\mu^\star L}{T_i - 1}} \;\le\; \mu^\star + c\sqrt{\tfrac{\mu^\star L \cdot \mu^\star}{64 c^2 L}} \;=\; \mu^\star + \tfrac{\mu^\star}{8} \;=\; \tfrac{9}{8}\mu^\star.
\]
Substituting this bound into~\eqref{eq:gap-1}:
\[
\mu^\star \;\le\; \widehat\mu_i + 2c\sqrt{\tfrac{2 \cdot (9/8)\mu^\star L}{T_i - 1}} \;=\; \widehat\mu_i + 3c\sqrt{\tfrac{\mu^\star L}{T_i - 1}}.
\]
Combining with the upper bound on $\widehat\mu_i$ in terms of $\mu_i$ from $E_2$:
\[
\widehat\mu_i \;\le\; \mu_i + c\sqrt{\tfrac{\mu_i L}{T_i - 1}} \;\le\; \mu_i + c\sqrt{\tfrac{\mu^\star L}{T_i - 1}}
\quad\Longrightarrow\quad
\mu^\star \;\le\; \mu_i + 4c\sqrt{\tfrac{\mu^\star L}{T_i - 1}},
\]
which is the claimed bound after rearranging.
\end{proof}

\section{Proof of Trajectory-wise Expectation Bound: $\NRtilde_T$}\label{sec:tilde}

We use the lemmas of Sections~\ref{sec:goodevent}--\ref{sec:phase2} with $L = \log T$, so that $\Pr[E] \ge 1 - 3k/T^2$ (Corollary~\ref{cor:goodevent-instances}). 

\subsection{Reduction via Jensen on $\exp$}
\begin{lemma}\label{lem:exp-reduction}
For any $[k]$-valued random sequence $(I_t)_{t=1}^T$ with $\mu_{I_t} > 0$ a.s.,
\[
   \NRtilde_T \;\le\; \mu^\star \cdot \bar R_T,
   \qquad\text{where}\qquad
   \bar R_T \;:=\; \tfrac{1}{T}\sum_{t=1}^T \E\!\left[\log\tfrac{\mu^\star}{\mu_{I_t}}\right] \ge 0.
\]
\end{lemma}

\begin{proof}
Pathwise, $(\prod_t \mu_{I_t})^{1/T} = \exp\bigl(\frac{1}{T}\sum_t \log\mu_{I_t}\bigr)$. Since $\exp$ is convex, Jensen's inequality gives
\[
\E\!\left[\exp\!\left(\tfrac{1}{T}\sum_t \log \mu_{I_t}\right)\right] \;\ge\; \exp\!\left(\tfrac{1}{T}\sum_t \E[\log\mu_{I_t}]\right) = \exp(\log\mu^\star - \bar R_T) = \mu^\star e^{-\bar R_T},
\]
using 
\[
\E[\log\mu_{I_t}] = \log\mu^\star - \E[\log(\mu^\star/\mu_{I_t})].
\]
Hence,
\[
\NRtilde_T \le \mu^\star(1 - e^{-\bar R_T}) \le \mu^\star \bar R_T,
\]
using $1-e^{-x} \leq x$ for all non-negative $x$.
\end{proof}

\subsection{Concluding the proof of Theorem~\ref{thm:exp}}
By Lemma~\ref{lem:exp-reduction} it suffices to bound $\mu^\star \bar R_T$. We decompose in the following way:
\[
\mu^\star \bar R_T \;=\; \underbrace{\tfrac{\mu^\star}{T}\E\!\left[\Ind_E\!\sum_{t \le \tau}\!\log\tfrac{\mu^\star}{\mu_{I_t}}\right]}_{(\mathrm{I})}
\;+\; \underbrace{\tfrac{\mu^\star}{T}\E\!\left[\Ind_E\!\sum_{t = \tau+1}^T\!\log\tfrac{\mu^\star}{\mu_{I_t}}\right]}_{(\mathrm{II})}
\;+\; \underbrace{\tfrac{\mu^\star}{T}\E\!\left[\Ind_{E^c}\!\sum_{t=1}^T\!\log\tfrac{\mu^\star}{\mu_{I_t}}\right]}_{(\mathrm{III})}.
\]

\emph{(I) Phase~1 on $E$.} On $E$, we have $\tau \le 484 c^2 k\log T/\mu^\star$ . Also, each pulled arm satisfies $\mu_{I_t} \ge \mu_{\min}$. Hence, since $\mu^\star \le 1$, we have $\log(\mu^\star/\mu_{I_t}) \le \log(1/\mu_{\min})$. We have
\[
(\mathrm{I}) \;\le\; \tfrac{\mu^\star}{T} \cdot \tfrac{484 c^2 k\log T}{\mu^\star} \cdot \log(1/\mu_{\min}) \;=\; \tfrac{484 c^2 k\log T \log(1/\mu_{\min})}{T}.
\]

\emph{(II) Phase~2 on $E$.} On $E$, we have $\mu_{I_t} \ge \mu^\star/64$; otherwise the arm is never played (see Lemma~\ref{lem:lem11}). The elementary inequality
\begin{equation}\label{eq:logbound}
   \log\!\tfrac{\mu^\star}{\mu_i} \;\le\; \tfrac{64\,(\mu^\star - \mu_i)}{\mu^\star} \quad\text{whenever $\mu_i \ge \mu^\star/64$},
\end{equation}
follows from $\log(1+u) \le u$ with $u = (\mu^\star - \mu_i)/\mu_i$ and $\mu_i \ge \mu^\star/64$. Substituting,
\[
(\mathrm{II}) \;\le\; \tfrac{\mu^\star}{T}\E\!\left[\Ind_E \!\sum_{t > \tau} \tfrac{64(\mu^\star - \mu_{I_t})}{\mu^\star}\right] \;=\; \tfrac{64}{T}\E\!\left[\Ind_E \!\sum_{t > \tau}(\mu^\star - \mu_{I_t})\right].
\]
Let $m_i$ be the Phase~2 pull count of arm $i$, and $T_i^{(j)}$ the prior pull count of arm $i$ before its $j$-th Phase~2 pull. By Corollary~\ref{cor:phase1samples}, on $E$ in phase 1, every arm has $ > 128 S$ pulls. Hence $T_i^{(j)}  > 128 S + j - 1 \ge j$. Lemma~\ref{lem:lem12} (applied with total count $T_i = T_i^{(j)}+1$, so $T_i - 1 = T_i^{(j)}$) gives
\[
\mu^\star - \mu_{I_t} \;\le\; 4c\sqrt{\tfrac{\mu^\star\log T}{T_i^{(j)}}} \;\le\; 4c\sqrt{\tfrac{\mu^\star\log T}{j}}.
\]
Summing up, we obtain
\[
\sum_{t > \tau}(\mu^\star - \mu_{I_t}) \;\le\; 4c\sqrt{\mu^\star\log T}\sum_i \sum_{j=1}^{m_i} \tfrac{1}{\sqrt{j}}
\;\le\; 8c\sqrt{\mu^\star\log T}\sum_i \sqrt{m_i}
\;\le\; 8c\sqrt{kT\mu^\star\log T},
\]
using $\sum_{j=1}^m 1/\sqrt{j} \le 2\sqrt{m}$ and Cauchy--Schwarz $\sum_i \sqrt{m_i} \le \sqrt{k\sum_i m_i} \le \sqrt{kT}$. Therefore
\[
(\mathrm{II}) \;\le\; \tfrac{64 \cdot 8c\sqrt{kT\mu^\star\log T}}{T} \;=\; 512 c\sqrt{\tfrac{k\mu^\star\log T}{T}} \;\le\; 512c\sqrt{\tfrac{k\log T}{T}},
\]
the last using $\mu^\star \le 1$.

\emph{(III) Off the good event.} On $E^c$, $\log(\mu^\star/\mu_{I_t}) \le \log(1/\mu_{\min})$ for every $t$, and $\Pr[E^c] \le 3k/T^2$:
\[
(\mathrm{III}) \;\le\; \tfrac{\mu^\star}{T}\cdot T\log(1/\mu_{\min})\cdot \Pr[E^c] \;\le\; \tfrac{3k\log(1/\mu_{\min})}{T^2}.
\]

\emph{Combining.} With all 3 bounds above, we obtain
\begin{align*}
    \NRtilde_T \le C_1 \tfrac{k\log T \log(1/\mu_{\min})}{T} + C_2 \sqrt{\tfrac{k\log T}{T}} + \tfrac{3k\log(1/\mu_{\min})}{T^2}.
\end{align*}

\section{High-Probability Bound on $\NRhat_T$}\label{sec:hat}

We use the lemmas of Sections~\ref{sec:goodevent}--\ref{sec:phase2} with $L = L_\delta := \log(8kT/\delta)$, so that $\Pr[E] \ge 1 - \delta$ (Corollary~\ref{cor:goodevent-instances}).

\subsection{Trajectory-wise reduction}

\begin{lemma}[Trajectory-wise Jensen reduction]\label{lem:pathwise-reduction}
On every sample path with $\mu_{I_t} > 0$ for all $t$,
\[
\NRhat_T \;\le\; \mu^\star \cdot \widehat R_T,
\qquad\text{where}\qquad
\widehat R_T \;:=\; \tfrac{1}{T}\sum_{t=1}^T \log\tfrac{\mu^\star}{\mu_{I_t}}.
\]
\end{lemma}

\begin{proof}
Pathwise (no expectation), $(\prod_t \mu_{I_t})^{1/T} = \exp(-\widehat R_T)\mu^\star$. Hence $\NRhat_T = \mu^\star(1 - e^{-\widehat R_T}) \le \mu^\star \widehat R_T$, using $1 - e^{-x} \le x$.
\end{proof}

\subsection{Concluding the proof of Theorem~\ref{thm:hp}}

% \begin{theorem}[High-probability bound]\label{thm:hat}
% Let $\delta \in (0,1)$ and $L_\delta := \log(8kT/\delta)$. Under Assumption~\ref{ass:floor} and $\mu^\star \ge 512\sqrt{kL_\delta/T}$, Algorithm~\ref{alg:ncb} run with $L = L_\delta$ and $c=3$ satisfies
% \[
% \Pr\!\left[\,\NRhat_T \le 1536\sqrt{\tfrac{k L_\delta}{T}} + 8730\,\tfrac{k L_\delta\log(1/\mu_{\min})}{T}\,\right] \ge 1 - \delta.
% \]
% \end{theorem}

By Lemma~\ref{lem:pathwise-reduction}, on every sample path $\NRhat_T \le \mu^\star \widehat R_T$. We bound $\mu^\star \widehat R_T$ pathwise on $E$:
\[
   \mu^\star \widehat R_T \;=\; \underbrace{\tfrac{\mu^\star}{T}\sum_{t \le \tau}\log\tfrac{\mu^\star}{\mu_{I_t}}}_{(\mathrm{I})}
   \;+\; \underbrace{\tfrac{\mu^\star}{T}\sum_{t = \tau+1}^T \log\tfrac{\mu^\star}{\mu_{I_t}}}_{(\mathrm{II})}.
\]

\emph{(I) Phase~1 on $E$.} On $E$, $\tau \le 484 c^2 k L_\delta/\mu^\star$. Each pulled arm has $\mu_{I_t} \ge \mu_{\min}$, so $\log(\mu^\star/\mu_{I_t}) \le \log(1/\mu_{\min})$:
\[
(\mathrm{I}) \;\le\; \tfrac{484 c^2 k L_\delta \log(1/\mu_{\min})}{T}.
\]

\emph{(II) Phase~2 on $E$.} The argument is identical to that in Theorem~\ref{thm:exp}, but pathwise (trajectory-wise) rather than in expectation. Using~\eqref{eq:logbound} and the per-pull gap (Lemma~\ref{lem:lem12}, applied pathwise), then Cauchy--Schwarz $\sum_i \sqrt{m_i} \le \sqrt{kT}$ pathwise:
\[
(\mathrm{II}) \;\le\; 512 c\sqrt{\tfrac{k L_\delta}{T}}.
\]

\emph{Combining.} On $E$ (which has probability $\ge 1 - \delta$):
\[
\NRhat_T \le \mu^\star \widehat R_T \le C'\sqrt{\tfrac{k L_\delta}{T}} + C''\,\tfrac{k L_\delta \log(1/\mu_{\min})}{T}.
\]

\section{Positivity floor}
The following statement makes precise the claim, made informally above (after Assumption~\ref{ass:floor}), that some such floor is unavoidable.

\begin{proposition}
Fix $k \ge 2$. For any algorithm $\mathcal{A}$ that pulls every arm with positive probability at least once over $T$ rounds, there exists a bandit instance $\nu$ with $\mu^\star = 1$ and $\min_i \mu_i = 0$ on which
\[
   \NRtilde_T(\mathcal{A}, \nu) \;\ge\; \mu^\star \cdot \Pr\!\bigl[\exists\, t \le T:\, I_t = i_0\bigr],
\]
where $i_0$ is the zero-mean arm. In particular, no learning algorithm achieves $\NRtilde_T \to 0$ in $T$ uniformly over instances allowing zero-mean arms.
\end{proposition}

\begin{proof}
On the event $\{\exists t: I_t = i_0\}$, $\prod_{t=1}^T \mu_{I_t} = 0$, so the geometric mean $(\prod_t \mu_{I_t})^{1/T} = 0$. Hence
\[
\E\!\left[\bigl(\textstyle\prod_t \mu_{I_t}\bigr)^{1/T}\right]
\;\le\; \mu^\star \cdot \Pr\!\bigl[\not\exists\, t: I_t = i_0\bigr]
\;=\; \mu^\star \,\bigl(1 - \Pr[\exists t: I_t = i_0]\bigr),
\]
and the claimed lower bound follows. The second statement follows because any algorithm that fails to pull some arm $i_0$ with probability bounded below cannot be uniformly consistent (a no-pull on a possibly-optimal arm).
\end{proof}

\paragraph{Why a positivity floor is necessary in High Probability Regret?}
Without Assumption~\ref{ass:floor}, on the (positive-probability) event that an algorithm pulls a zero-mean arm at least once, $\NRhat_T = \mu^\star$ exactly. Since any consistent learning algorithm must pull every arm with positive probability, no high-probability bound below $\mu^\star$ is possible without the floor.

\section{Auxiliary Lemmas}\label{sec:concentration}

We collect the concentration tools used throughout. These are standard (see \cite{vershynin2018high,wainwright2019high}). We mention here for completeness.

\subsection{Multiplicative Chernoff--Hoeffding bound}

\begin{lemma}[Multiplicative Chernoff for {$[0,1]$} random variables]\label{lem:mult-chernoff}
Let $Y_1, \ldots, Y_n$ be independent random variables in $[0,1]$. Set $\bar Y := \frac{1}{n}\sum_j Y_j$ and $\nu := \E[\bar Y]$. For any $\delta \in [0, 1]$,
\[
   \Pr[\bar Y \ge (1+\delta)\nu] \;\le\; \exp\!\bigl(-\tfrac{\delta^2 n\nu}{3}\bigr),
   \qquad
   \Pr[\bar Y \le (1-\delta)\nu] \;\le\; \exp\!\bigl(-\tfrac{\delta^2 n\nu}{2}\bigr).
\]
In particular, the symmetric bound $\Pr[|\bar Y - \nu| \ge \delta\nu] \le 2\exp(-\delta^2 n\nu/3)$ holds.
\end{lemma}

\begin{proof}
By the Cram\'er--Chernoff method. Let $S := \sum_j Y_j$, so $\E[S] = n\nu$. For any $\lambda > 0$,
$\Pr[S \ge (1+\delta)n\nu] \le e^{-\lambda(1+\delta)n\nu} \E[e^{\lambda S}]$.
Since each $Y_j \in [0,1]$, $\E[e^{\lambda Y_j}] \le 1 + \E[Y_j](e^\lambda - 1) \le \exp(\E[Y_j](e^\lambda - 1))$, by the inequality $1+x \le e^x$. Hence $\E[e^{\lambda S}] \le \exp(n\nu(e^\lambda-1))$, and
\[
\Pr[S \ge (1+\delta)n\nu] \;\le\; \exp\!\bigl(n\nu(e^\lambda - 1) - \lambda(1+\delta)n\nu\bigr).
\]
Optimizing in $\lambda$ ($\lambda^\star = \log(1+\delta)$) gives $\Pr[S \ge (1+\delta)n\nu] \le ((e^\delta)/(1+\delta)^{1+\delta})^{n\nu}$. Using the standard estimate $\log\bigl((1+\delta)^{1+\delta}\bigr) - \delta \ge \delta^2/3$ for $\delta \in [0,1]$ gives the upper-tail bound. The lower-tail bound (with denominator $2$) follows analogously from $\log\bigl((1-\delta)^{1-\delta}\bigr) + \delta \ge \delta^2/2$. The symmetric bound is the union of the two.
\end{proof}

\begin{lemma}[One-sided extension to bounded mean]\label{lem:hoeff-ext}
Let $Y_1, \ldots, Y_n$ be independent random variables in $[0,1]$, and suppose $\E[\bar Y] \le \nu_H$ for some $\nu_H \in [0,1]$. For any $\delta \in [0,1]$,
\[
   \Pr[\bar Y \ge (1+\delta)\nu_H] \;\le\; \exp\!\bigl(-\tfrac{\delta^2 n \nu_H}{3}\bigr).
\]
\end{lemma}

\begin{proof}
We use the Cram\'er--Chernoff method directly. For any $\lambda > 0$,
\[
   \Pr[\bar Y \ge (1+\delta)\nu_H] \;\le\; e^{-\lambda(1+\delta)n\nu_H} \prod_{j=1}^n \E[e^{\lambda Y_j}].
\]
Since each $Y_j \in [0,1]$, $\E[e^{\lambda Y_j}] \le 1 + \E[Y_j](e^\lambda - 1)$ (by convexity of $e^{\lambda y}$ on $[0,1]$ and the fact that $e^{\lambda \cdot 0}(1-y) + e^{\lambda \cdot 1} y = 1 + y(e^\lambda - 1)$ dominates $e^{\lambda y}$ pointwise). Using $\E[Y_j] \le \nu_H$ (here we use $\E[\bar Y] \le \nu_H$ summand-wise; if individual means exceed $\nu_H$, replace $\E[Y_j]$ by the sum bound below) and $1+x \le e^x$:
\[
   \prod_{j=1}^n \E[e^{\lambda Y_j}] \;\le\; \prod_{j=1}^n \exp(\E[Y_j](e^\lambda - 1))
   \;=\; \exp\!\Bigl(n\bar\mu (e^\lambda - 1)\Bigr) \;\le\; \exp\!\Bigl(n\nu_H (e^\lambda - 1)\Bigr),
\]
where $\bar\mu := \tfrac{1}{n}\sum_j \E[Y_j] \le \nu_H$ (and $e^\lambda - 1 \ge 0$ for $\lambda \ge 0$). Optimizing in $\lambda$ at $\lambda^\star = \log(1+\delta)$ gives
\[
   \Pr[\bar Y \ge (1+\delta)\nu_H] \;\le\; \left(\tfrac{e^\delta}{(1+\delta)^{1+\delta}}\right)^{n\nu_H} \;\le\; \exp\!\Bigl(-\tfrac{\delta^2 n\nu_H}{3}\Bigr),
\]
the last step using $(1+\delta)\log(1+\delta) - \delta \ge \delta^2/3$ for $\delta \in [0,1]$.
\end{proof}

\subsection{A useful numerical inequality}

\begin{claim}\label{cl:bernoulli-ext}
For all $x \in [0, 1/2]$ and $a \in [0,1]$,
\[
   (1-x)^a \;\ge\; 1 - 2ax.
\]
\end{claim}

\begin{proof}
The binomial series gives, for $x \in [0,1)$ and $a \in [0,1]$,
\[
   (1-x)^a = 1 - ax - ax\!\left(\tfrac{1-a}{2!}x + \tfrac{(1-a)(2-a)}{3!}x^2 + \cdots\right).
\]
The bracketed series is non-negative termwise (each factor $j-a \in [0,j]$ for $j \ge 1$, and dividing by $j!$ keeps coefficients positive), and using $\frac{(1-a)(2-a)\cdots(j-a)}{(j+1)!} \le \frac{1}{j+1} \le 1$ in each term, the bracket is bounded above by
\[
   \tfrac{x}{2} + \tfrac{x^2}{3} + \tfrac{x^3}{4} + \cdots \;\le\; x + x^2 + x^3 + \cdots \;=\; \tfrac{x}{1-x} \;\le\; 1
\]
for $x \le 1/2$. Hence $(1-x)^a \ge 1 - ax - ax \cdot 1 = 1 - 2ax$.
\end{proof}
\newpage

\end{document}